\documentclass[letterpaper]{article}
\usepackage{uai2020}
\usepackage[margin=1in]{geometry}

\usepackage{times}
\usepackage{natbib}
\usepackage{apalike}
\usepackage{url}
\usepackage{amsmath,amsthm,amssymb,bm}
\usepackage{bbold}
  
\DeclareMathOperator{\argmax}{argmax}
\usepackage{algorithm,algorithmic}
\usepackage{graphicx}
\usepackage{hyperref}

\newtheorem{thm}{Theorem}
\newtheorem{defini}{Definition}
\newtheorem{lemma}{Lemma}

\newtheorem{claim}{Claim}

\title{Regret Analysis of Bandit Problems with Causal Background Knowledge}

\author{ { \bf Yangyi Lu} \\
Department of Statistics \\
University of Michigan\\
{\tt yylu@umich.edu}
\And
{\bf Amirhossein Meisami}  \\
Adobe Inc.          \\
{\tt meisami@adobe.com} \\
\And
{\bf Ambuj Tewari}   \\
Department of Statistics \\
University of Michigan\\
{\tt tewaria@umich.edu}   \\
\And
{\bf Zhenyu Yan} \\
Adobe Inc.\\
{\tt wyan@adobe.com}
}

\begin{document}

\maketitle

\begin{abstract}
 We study how to learn optimal interventions sequentially given causal information represented as a causal graph along with associated conditional distributions.
Causal modeling is useful in real world problems like online advertisement where complex causal mechanisms underlie the relationship between interventions and outcomes.
  We propose two algorithms, causal upper confidence bound (C-UCB) and causal Thompson Sampling (C-TS), that enjoy improved cumulative regret bounds compared with algorithms that do not use causal information.
We thus resolve an open problem posed by~\cite{lattimore2016causal}.
Further, we extend C-UCB and C-TS to the linear bandit setting and propose causal linear UCB (CL-UCB) and causal linear TS (CL-TS) algorithms. These algorithms enjoy a cumulative regret bound that only scales with the feature dimension. 
Our experiments show the benefit of using causal information. For example, we observe that even with a few hundreds of iterations, the regret of causal algorithms is less than that of standard algorithms by a factor of three. We also show that under certain causal structures, our algorithms scale better than the standard bandit algorithms as the number of interventions increases.

\end{abstract}
\section{INTRODUCTION} \label{sec:intro}
In a multi-armed bandit (MAB) problem, an agent adaptively learns to pull arms from a finite set of arms based on the past knowledge. At each pull, it observes a single reward corresponding to the arm pulled and its goal is to maximize the cumulative reward received within a time horizon. Bandit models are widely used in various applications, such as education~\citep{williams2016axis}, clinical trials~\citep{villar2015multi,tewari2017ads} and marketing~\citep{burtini2015improving,mersereau2009structured}.

There are many well-studied stochastic bandit algorithms, such as upper confidence bound (UCB)~\citep{auer2002finite} and Thompson Sampling (TS)~\citep{agrawal2012analysis}, that can both achieve a regret bound $\tilde{O}(\sqrt{KT})$\footnote{$\tilde{O}$ ignores constant and poly-logarithmic factors.}, where $K$ is the number of arms and $T$ is the time horizon. 
However, in many real world applications where we search for good interventions, the number of actions (interventions) is extremely large. 
An intervention here is defined as a forcible change to the value of a set of variables.

As an example of a real world problem with a large space of available interventions, we focus on the email campaign problem. Online advertising companies are constantly looking for an optimal trade-off between exploration and exploitation efforts in order to convert a potential buyer to an actual buyer. 
In case of email campaigns, the overall target is to maximize the user interaction with the emails that could be defined as opening an email, clicking on a link or eventually buying a product. 
To achieve these goals, marketers adjust several variables in the process. 
For instance, they may know that the length of subject, the template, the time of day to send, the product and the type (promotion, online events, etc.) of an email can affect whether a customer who receives the email will click the links inside or not.
Every possible assignment of values to these variables can be an intervention leading to an extremely large number of interventions. Therefore, strategic utilization of such interventions is necessary for maximizing the cumulative user conversion throughout the campaign horizon.

A natural approach to deal with a large number of interventions is to exploit relationships between the way different interventions affect the outcome. 
In this paper, we focus on causal relations among interventions. 
In particular, we use causal graphs~\citep{pearl2000causality} to represent relationships between interacting variables in a complex system.
We study the following problem: using previously acquired knowledge about the causal graph structure, how to quickly learn good interventions sequentially~\citep{sen2017identifying,hyttinen2013experiment}?
Our goal is to optimize over a given set of interventions in a sequential decision making framework where the dependence among reward distribution of these interventions is captured through a causal structure.

\citet{lattimore2016causal} proposed two causal bandit algorithms, but they only provided simple regret guarantees and their bounds scale with the number of interventions in the worst case. 
Indeed, one of the open problems in their paper is to design algorithms that enjoy a $\tilde{O}(\sqrt{T})$ cumulative regret bound, and utilize the causal structure at the same time.
Cumulative regret is appropriate when both exploration and exploitation are needed, while simple regret is useful when it is important to identify a good intervention at the end of a pure exploration phase. In many real world problems, we are not simply looking for the best intervention as quickly as possible without consideration of outcomes obtained during the exploration phase. 
In email campaign or clinical trials problems, a good policy should lead to high revenue and conversions or good health outcomes cumulatively, which are not what a pure exploration method can achieve. 
Therefore we focus on cumulative regret in this paper. 

\subsection{OUR CONTRIBUTIONS}\label{sec:contribution}
We propose two natural and efficient algorithms, causal UCB (C-UCB) and causal TS (C-TS), by incorporating the available causal knowledge in UCB and TS for multi-armed bandit problems. We use causal knowledge to greatly reduce the amount of exploration needed to achieve low cumulative regret. 

Suppose there are $N$ variables that are related to the reward and each of them takes on $k$ distinct values, which means changing the value of any of these variables can affect the reward distribution. 
Note the number of interventions can be as large as $(k+1)^N$, which means that standard bandit algorithms are only guaranteed to achieve $\tilde{O}(\sqrt{(k+1)^NT})$ regret.
Our proposed causal algorithms exploit the causal knowledge to achieve $\tilde{O}(\sqrt{(k+1)^nT})$ regret\footnote{Our regret bounds for confidence bound based algorithms will be frequentist while for Thompson sampling they will be Bayesian.}, where $n$ is the number of variables that have {\em direct} causal effects on the reward.
These bounds suggest that causal UCB and TS algorithms are preferable to standard UCB and TS algorithms when $n\ll N$.

We further extend the causal bandit algorithms to linear bandit setting, that leads to our causal linear UCB (CL-UCB) and causal linear TS (CL-TS) algorithms.
We show that CL-UCB and CL-TS both achieve $\tilde{O}(d\sqrt{T})$ regret, where $d$ is the dimension of the coefficient vector in the linear reward model.

To complement our upper bounds, we also provide a lower bound for standard UCB algorithm. 
For some structured bandit instances with $n<N$, we show a lower bound on the cumulative regret of standard UCB which comes arbitrarily close to $\Omega(\sqrt{(k+1)^NT})$, which is much larger than the upper bounds of our proposed algorithms that utilize causal structures. It demonstrates that a standard MAB algorithm is {\em provably} worse than causal algorithms in certain cases.

Our experiments show the benefit of using causal structure: we observe (see Section~\ref{sec:experiments}, Figure~\ref{fig:Algo_comparison}) that within hundreds of iterations, our causal algorithms are already achieving regret within $1/3$ of the standard algorithms' regret.
In addition, we validate numerically that for certain causal graph structure, C-UCB, C-TS, CL-UCB and CL-TS indeed scale better than standard multi-armed bandit algorithms as the size of intervention set grows.  

\subsection{RELATED WORK}\label{sec:related_work}
Causal bandit problems can be treated as multi-armed bandit problems by simply ignoring the causal structure information and the extra observations. So existing bandit algorithms such as UCB~\citep{auer2002finite} and TS~\citep{agrawal2012analysis} can be applied. 
However, causal information should help us learn about an intervention based on the performance of other interventions, which can accelerate the whole learning process. 

Combinatorial bandits~\citep{cesa2012combinatorial} also deal with an action set that is exponentially large. For example, the action set is usually a subset of the $d$-dimensional binary hypercube. In combinatorial bandits, the goal is to feasibly learn in the large action space by assuming certain structure (e.g., linear) in the reward dependence on actions and the availability of an efficient optimization solver over the action set. However, our emphasis is on reducing the statistical complexity by exploiting the given causal structures.

We build on the work of~\citet{lattimore2016causal}. They studied the problem of identifying the best interventions in a stochastic bandit environment with known causal graph and some conditional probabilities of variables in the graph. 
They proposed two algorithms depending on the type of causal graphs: parallel graph/general graph, and proved two simple regret bounds accordingly. 
Both bounds scale with a measure for causal graph's underlying distribution, which is small if every intervention has similar effect on the reward and can be as large as the number of interventions otherwise.
Moreover, their algorithm for general graph contains as many parameters as the number of interventions, which are hard to tune.
We focus on the cumulative regret and our algorithms are universal for all directed acyclic causal graphs defined in Section~\ref{sec:setup} with no tuning parameters other than that of standard MAB algorithms.

Another work~\citep{sen2017identifying} also considered best intervention identification via importance sampling, and their interventions are soft. 
Instead of forcing a node to take a specific value, soft intervention only changes the conditional distribution of a node given its parent nodes.
However, they also only considered simple regret and their bounds scale with the number of interventions.
\cite{sachidananda2017} studied the most closest setting as our paper. 
They showed the effectiveness of their causal Thompson Sampling method, but did not provide any regret analysis. \citet{lee2018structural} empirically showed that a brute-force way to apply standard bandit algorithms on all interventions can suffer huge regret.
Therefore they proposed a way to carefully choose an intervention subset by observing the causal graph structures. Our lower bound (Theorem~\ref{thm:UCB_lower}) provides a theoretical explanation for the phenomenon they observe, namely that brute-force algorithms that try all possible interventions can incur huge regret.


\section{PROBLEM SETUP} \label{sec:setup}

We follow standard terminology and notation~\citep{koller2009probabilistic} to state the causal bandit problem introduced by~\citet{lattimore2016causal}.
A directed acyclic graph $\mathcal{G}$ is used to model the causal structure over a set of random variables $\mathcal{X} = \{X_1,\ldots,X_N\}$. Let $P$ denote the joint distribution over $\mathcal{X}$ that factorizes over $\mathcal{G}$. 
For simplicity, we assume each variable can take on $k$ distinct values, but extending our algorithm to various $k$ values for different variables poses no difficulty. The parents of a variable $X_i$, denoted by $\text{Pa}_{X_i}$, is the set of all variables $X_j$ such that there is an edge from $X_j$ to $X_i$ in graph $\mathcal{G}$. 
A size $m$ intervention (action) is denoted by $\text{do}(\mathbf{X} = \mathbf{x})$, which assigns the values $\mathbf{x} = \{x_1,\ldots,x_m\}$ to the corresponding variables $\mathbf{X} = \{X_1,\ldots,X_m\} \subset \mathcal{X}$. 
An empty intervention is $\text{do}()$. 
The intervention on $\mathbf{X}$ also removes all edges from $\text{Pa}_{X_i}$ to $X_i$ for each $X_i \in \mathbf{X}$. 
Thus the resulting underlying probability distribution that defines the graph is denoted by $P(\mathbf{X}^c|\text{do}(\mathbf{X}=\mathbf{x}))$ over $\mathbf{X}^c:= \mathcal{X}\setminus\mathbf{X}$.

In this causal bandit problem, the reward variable $Y$ is real-valued. 
A learner is given the causal model's graph $\mathcal{G}$\footnote{Even though our algorithms take $\mathcal{G}$ as input, the only information used is the identity of $\text{Pa}_Y$ variables.}, a set of interventions (actions) $\mathcal{A}$ and conditional distributions of parent variables of $Y$ given an intervention $a \in \mathcal{A}$: $P(\text{Pa}_Y|a)$. 
We denote the expected reward for action $a = \text{do}(\mathbf{X}=\mathbf{x})$ and the optimal action $a^*$ by:
\begin{align*}
	\mu_a&:= \mathbb{E}\left[Y|\text{do}(\mathbf{X}=\mathbf{x})\right]\\
	a^* &:= \argmax_{a\in \mathcal{A}}\mu_a.
\end{align*}
We assume $\mu_a \in [0,1]$ for every $a \in \mathcal{A}$. In round $t$, the learner pulls $a_t = \text{do}(\mathbf{X}_t=\mathbf{x}_t)$ based on previous round knowledge and causal information, then observes the reward $Y_t$ and the values of $\text{Pa}_Y$, denoted by $\mathbf{Z}_{(t)} = \{z_1(t),\ldots,z_n(t)\}$, where $n$ is the number of reward's parent variables. 
However, in the work of \citet{lattimore2016causal}, they need to observe the values of all variables after taking an action. Thus, comparing to them, the problem we face is more challenging.
We know there are $k^n$ different value assignments on $\text{Pa}_Y$, for convenience, we denote them by $\mathbf{Z}_1,\ldots,\mathbf{Z}_{k^n}$, where each $\mathbf{Z}_i$ is a vector of length $n$.

The objective of the learner is to minimize the expected cumulative regret $\mathbb{E}\left[R_T\right] = T\mu^* - \sum_{t=1}^{T}\mathbb{E}\left[\mu_{a_t}\right]$ using causal knowledge. 

\textbf{Bayesian Regret:} Let $\omega \in \Omega$ denote the entire parameters of the distribution of $Y|_{\text{Pa}_Y=\mathbf{Z}}$. Reward can be expressed by $Y = \mathbb{E}\left[Y|_{\text{Pa}_Y=\mathbf{Z}}\right]+\epsilon$, where $\epsilon$ is a 1-subgaussian error variable. Thus, the cumulative regret $R_T$ for a given $\omega$ can be formally written as $R_T(\omega)$. We particularly focus on the case where $\omega$ is random with distribution $Q$ and bound the following Bayesian regret:
	$BR_T = \mathbb{E}_{\omega \sim Q} \mathbb{E}_\epsilon R_T(\omega)$.

\textbf{Worst Case Regret:} Using same notations as above, the worst case (frequentist) regret is defined as:
   $ \max_{\omega\in\Omega}\mathbb{E}_\epsilon R_T(\omega)$.
$\mathbb{E}R_T$ is used to represent the worst case regret from now for short.

\section{CAUSAL BANDIT ALGORITHMS} \label{sec:algo}
In this section we propose and analyze algorithms for achieving minimal regret when causal information is known. We generalize standard UCB and standard TS algorithms to their causal counterparts in a natural way. We show how the regret bounds of the causal versions scale with a factor that can be much smaller than what would be the case for the standard algorithms. We also extend linear bandit algorithms to their causal version and demonstrate how it further helps us reduce the cumulative regret.

\subsection{CAUSAL MAB ALGORITHMS}
In the first part of this section we consider causal MAB problem and present causal upper confidence bound algorithm (C-UCB) and causal Thompson Sampling algorithm (C-TS).
\subsubsection{Causal UCB (C-UCB)}
Without causal knowledge, UCB algorithm updates the confidence interval of the reward mean for each arm. At every round, the learner chooses the arm with the highest upper confidence bound value. However, thanks to causal graph structures, we are able to make use of the expectation decomposition formula
\begin{align*}
	\mu_a = \sum_{j=1}^{k^n}\mathbb{E}\left[Y|\text{Pa}_Y=\mathbf{Z}_j\right]P(\text{Pa}_Y=\mathbf{Z}_j|a).
\end{align*}
At every round $t$, Algorithm~\ref{algo:UCB_knownGraph} only updates the reward mean and upper confidence bound for every possible value assignment on reward's parent variables denoted by $\text{UCB}_{\mathbf{Z}_j}(t-1)$ as $P(\text{Pa}_Y=\mathbf{Z}_j|a)$ terms are known. It provides the upper confidence bound for each arm by:
\begin{align*}
	\textbf{UCB}_a(t-1) = \sum_{j=1}^{k^n}\text{UCB}_{\mathbf{Z}_j}(t-1)P(\text{Pa}_Y=\mathbf{Z}_j|a).
\end{align*}
We pull $a_t$ that can maximize $\textbf{UCB}_a(t-1)$ over all $a\in\mathcal{A}$.
There remain fewer upper confidence bounds to construct since usually $k^n<(k+1)^N$, so it is reasonable to expect that the cumulative regret can be reduced. 

\begin{algorithm}[tb]
	\caption{C-UCB}
	\label{algo:UCB_knownGraph}
	\begin{algorithmic}
		\STATE \textbf{Input:} Horizon $T$, action set $\mathcal{A}$, $\delta$, causal graph $\mathcal{G}$, number of parent variables $n$, number of values each parent variable can take on: $k$. 
		\STATE \textbf{Initialization:}
		Values assignment to parent variables: $\mathbf{Z}_j$, $\hat{\mu}_{\mathbf{Z}_j}(0)=0$, $T_{\mathbf{Z}_j}(0)=0$, for $j=1,\ldots,k^n$.
		\FOR{$t=1,\ldots,T$}
		\FOR{$j=1,\ldots,k^n$}
		\STATE $\text{UCB}_{\mathbf{Z}_j}(t-1) = \hat{\mu}_{\mathbf{Z}_j}(t-1)+\sqrt{\frac{2\log(1/\delta)}{1\vee T_{\mathbf{Z}_j}(t-1)}}$.
		\ENDFOR
		\STATE $a_t = \argmax_{a\in\mathcal{A}} \sum_{j=1}^{k^n}\text{UCB}_{\mathbf{Z}_j}(t-1)P(\text{Pa}_Y=\mathbf{Z}_j|a)$
		\STATE Pull arm $a_t$ and observe reward $Y_t$ and its parent nodes' values $\mathbf{Z}_{(t)}$.
		\STATE Update $T_{\mathbf{Z}_j}(t) = \sum_{s=1}^{t} \mathbb{1}_{\{\mathbf{Z}_{(s)}=\mathbf{Z}_j\}}$ and $\hat{\mu}_{\mathbf{Z}_j}(t) = \frac{1}{T_{\mathbf{Z}_j}(t)}\sum_{s=1}^{t}Y_s\mathbb{1}_{\{\mathbf{Z}_{(s)}=\mathbf{Z}_j\}}$, for $j=1,\ldots,k^n$.
		\ENDFOR
	\end{algorithmic}
\end{algorithm}

\begin{thm}[Regret Bound for C-UCB] \label{thm:C-UCB}
	Let $Y|_{\text{Pa}_Y=\mathbf{Z}_j} = \mathbb{E}\left[Y|\text{Pa}_Y=\mathbf{Z}_j\right]+\epsilon$, for $j=1,\ldots,k^n$, where $\epsilon$ is a mean zero, $1$-subgaussian distributed random error. If $\delta=1/T^2$, the regret of policy defined in Algorithm~\ref{algo:UCB_knownGraph} is bounded by
	\begin{align*}
	\mathbb{E}\left[R_T\right] = \tilde{O}\left(\sqrt{k^nT}\right).
	\end{align*}
\end{thm}

\subsubsection{Causal TS (C-TS)}
Thompson Sampling (TS) algorithm needs to update the posterior distributions for all arms. 
In this problem, there are $(k+1)^N$ distributions to update, which leads to huge regret when $N$ is large.
Similar to UCB algorithms, causal information can greatly help TS improve the performance when $k^n < (k+1)^N$. 
Again, by using the expectation decomposition formula $\mu_a = \sum_{j=1}^{k^n}\mathbb{E}\left[Y|\text{Pa}_Y=\mathbf{Z}_j\right]P(\text{Pa}_Y=\mathbf{Z}_j|a)$, our C-TS algorithm only updates the posterior distributions for $Y|_{\text{Pa}_Y=\mathbf{Z}_j}, j=1,\ldots,k^n$ as the $P(\text{Pa}_Y=\mathbf{Z}_j|a)$ terms are known.

We provide two C-TS algorithms where Algorithm~\ref{algo:TS_Beta} uses Beta distribution as its prior and Algorithm~\ref{algo:TS_Gaussian} uses Gaussian distribution as its prior. 
At every round t, both C-TS algorithms sample from the posterior distributions for $Y|_{\text{Pa}_Y=\mathbf{Z}_j}, j=1,\ldots,k^n$, then construct the estimated reward mean denoted by $\hat{\mu}_a$ for $\forall a\in \mathcal{A}$ using causal information. The intervention arm with the highest estimated reward will be pulled, reward $Y_t$ and parent node values $\mathbf{Z}_{(t)}$ will be revealed accordingly. Parameters for Beta or Gaussian distribution are updated according to Beta-Bernoulli and Gaussian-Gaussian prior-posterior updating formulas.
\begin{algorithm}[tb]
	\caption{C-TS with Beta Prior (If $Y \in [0,1]$)}
	\label{algo:TS_Beta}
	\begin{algorithmic}
		\STATE \textbf{Input:} Horizon $T$, action set $\mathcal{A}$, causal graph $\mathcal{G}$, all $P(Pa_Y|a)$, number of parent variables $n$, number of values each parent variable can take on: $k$.
		\STATE \textbf{Initialization:} Value assignments to parent variables: $\mathbf{Z}_j$, $S_{\textbf{Z}_j}^0 = F_{\textbf{Z}_j}^0 = 1$, for $j = 1,\ldots,k^n$.
		\FOR{$t\in \{1,\ldots,T\}$}
		\STATE Sample $\hat{\theta}_j(t)$ from beta distn with parameters $(S_{\mathbf{Z}_j}^{t-1},F_{\mathbf{Z}_j}^{t-1})$, for $j = 1,\ldots,k^n$.
		\FOR{action $a\in\mathcal{A}$}
		\STATE $\hat{\mu}_a = \sum_{j=1}^{k^n}\hat{\theta}_j(t)P(\text{Pa}_Y=\mathbf{Z}_j|a)$
		\ENDFOR
		\STATE $a_t = \argmax_a \hat{\mu}_a$
		\STATE Pull arm $a_t$ and observe reward $\tilde{Y}_t$ and its parent nodes values of $\mathbf{Z}_{(t)}$. Perform a Bernoulli trial with success probability $Y_t$ and observe the output $Y_t$.
		\IF{$Y_t=1$}
		\STATE $S_{\mathbf{Z}_{(t)}}^t = S_{\mathbf{Z}_{(t)}}^{t-1}+1$
		\ELSE
		\STATE $F_{\mathbf{Z}_{(t)}}^t = F_{\mathbf{Z}_{(t)}}^{t-1}+1$
		\ENDIF
		\ENDFOR
	\end{algorithmic}
\end{algorithm}

\begin{algorithm}[tb]
	\caption{C-TS with Gaussian Prior}
	\label{algo:TS_Gaussian}
	\begin{algorithmic}
		\STATE \textbf{Input:} Horizon $T$, action set $\mathcal{A}$, causal graph $\mathcal{G}$, all $P(Pa_Y|a)$, number of parent variables $n$, number of values each parent variable can take on: $k$.
		\STATE \textbf{Initialization:} Value assignments to parent variables: $\mathbf{Z}_j$, $k_{\mathbf{Z}_j}=0$, $\hat{\mu}_{\mathbf{Z}_j}=0$, for $j = 1,\ldots,k^n$.
		\FOR{$t\in \{1,\ldots,T\}$}
		\STATE Sample $\hat{\theta}_j(t)$ $\sim$ $N(\hat{\mu}_{\mathbf{Z}_j},\frac{1}{k_{\mathbf{Z}_j}+1})$ ,for $j = 1,\ldots,k^n$.
		\FOR{action $a\in\mathcal{A}$}
		\STATE $\hat{\mu}_a = \sum_{j=1}^{k^n}\hat{\theta}_j(t)P(\text{Pa}_Y=\mathbf{Z}_j|a)$
		\ENDFOR
		\STATE $a_t = \argmax_a \hat{\mu}_a$
		\STATE Pull arm $a_t$ and observe the parent nodes values of Y denoted by $\mathbf{Z}_{(t)}$ and reward $Y_t$. 
		\STATE Update $k_{\mathbf{Z}_t}: = k_{\mathbf{Z}_t}+1$
		\STATE Update $\hat{\mu}_{\mathbf{Z}_t}: = \frac{\hat{\mu}_{\mathbf{Z}_t}k_{\mathbf{Z}_t}+Y_t}{k_{\mathbf{Z}_t}+1}$
		\ENDFOR
	\end{algorithmic}
\end{algorithm}

\begin{thm}[Bayesian Regret Bound for C-TS]\label{thm:TS_bayes}
	Let $Y|_{\text{Pa}_Y=\mathbf{Z}_j} = \mathbb{E}\left[Y|\text{Pa}_Y=\mathbf{Z}_j\right]+\epsilon$, for $j=1,\ldots,k^n$, where $\epsilon$ is a mean zero, $1$-subgaussian distributed random error. Then the Bayesian regret of policies in Algorithm~\ref{algo:TS_Beta} and Algorithm~\ref{algo:TS_Gaussian} are both be bounded by:
	\begin{align*}
	BR_T = \tilde{O}\left(\sqrt{k^nT}\right).
	\end{align*}
\end{thm}


\subsection{CAUSAL LINEAR BANDIT ALGORITHMS}
Previous section demonstrates how we use causal knowledge to improve the multi-armed bandit algorithms. In our setting, the reward $Y$ directly depends on its $n$ parent nodes, then a natural extension is to consider the linear modeling case:
$Y|_{\text{Pa}_Y=\mathbf{Z}} = f(\mathbf{Z})^T\theta + \epsilon$,
where $f$ denotes the feature function applied on the parent nodes of $Y$, $\theta$ denotes the linear coefficient and $\epsilon$ is a zero mean, 1-subgaussian distributed random error.

We can write the expected reward mean for $\forall a\in\mathcal{A}$ as:
\begin{align*}
\mu_a 
&= \langle \sum_{j=1}^{k^n}f(\mathbf{Z}_j)P(\text{Pa}_Y=\mathbf{Z}_j|a),\theta\rangle.
\end{align*}
To this point, we demonstrate that linearly modeling the reward's parent nodes is just a special case of standard linear bandit problem, where the feature vector for $a\in\mathcal{A}$ is $m_a:=\sum_{j=1}^{k^n}f(\mathbf{Z}_j)P(\text{Pa}_Y=\mathbf{Z}_j|a)$. Thus, we easily extend C-UCB and C-TS to this particular linear bandit setting.

Causal linear UCB (CL-UCB) algorithm (Algorithm~\ref{algo:CL-UCB}) and causal linear TS (CL-TS) algorithm (Algorithm~\ref{algo:CL-TS}) are straightforward linear UCB and linear TS algorithms. It is helpful in the sense that the regret dependence on $\sqrt{k^n}$ can be further reduced to the dimension of linear coefficient $\theta$ denoted by $d$ while linear reward over parent variables holds.

\begin{algorithm}[h]
	\caption{Causal Linear UCB (CL-UCB)}
	\label{algo:CL-UCB}
	\begin{algorithmic}
		\STATE \textbf{Input:} horizon $T$, action set $\mathcal{A}$, all $P(Pa_Y|a)$.
		\STATE \textbf{Initialization:} $V_0 = I_d$, $\hat{\theta}_0 = 0_d$, $g=0_d$, $\beta = 1+\sqrt{2\log\left(T\right)+d\log\left(1+\frac{T}{d}\right)}$.
		\FOR{$t=1,\ldots,T$}
		\FOR{$a\in \mathcal{A}$}
		\STATE $\text{UCB}_a(t) = \max_{\theta\in \mathcal{C}_t} \langle \theta,m_a\rangle = \langle \hat{\theta}_{t-1},m_a\rangle+\beta\left\|m_a\right\|_{V_{t-1}^{-1}}$, where $\mathcal{C}_t = \left\{\theta\in \mathbb{R}^d: \left\|\theta-\hat{\theta}_{t-1}\right\|_{V_{t-1}} \leq \beta \right\}$
		\ENDFOR
		\STATE $a_t = \argmax_{a\in\mathcal{A}} \text{UCB}_a(t)$
		\STATE Pull arm $a_t$ and observe parent node $Z_{(t)}$ and reward $Y_t$.
		\STATE Update $V_t=V_{t-1}+m_{a_t}m_{a_t}^T$, $g=g+m_{a_t}Y_t$, $\hat{\theta}_t = V_t^{-1}g$
		\ENDFOR
	\end{algorithmic}
\end{algorithm}

\begin{algorithm}[h]
	\caption{Causal Linear TS (CL-TS)}
	\label{algo:CL-TS}
	\begin{algorithmic}
		\STATE \textbf{Input:} Horizon $T$, action set $\mathcal{A}$, all $P(Pa_Y|a)$, standard deviation parameter $v$.
		\STATE \textbf{Initialization:} $V_0 = I_d$, $\hat{\theta} = 0_d$, $g = 0_d$.
		\FOR{$t\in \{1,\ldots,T\}$}
		\STATE Sample $\tilde{\theta}_t \sim N(\hat{\theta},v^2V_t^{-1})$
		\FOR{action $a\in\mathcal{A}$}
		\STATE $\hat{\mu}_a(t) = 
		\langle m_a,\tilde{\theta}_t \rangle$
		\ENDFOR
		\STATE $a_t = \argmax_a \hat{\mu}_a(t)$
		\STATE Pull arm $a_t$ and observe the parent nodes values of Y denoted by $Z_{(t)}$ and reward $Y_t$.
		\STATE Update $V_t = V_{t-1}+m_{a_t}m_{a_t}^T$, $g = g+m_{a_t}Y_t$ and $\hat{\theta} = V_t^{-1}g$
		\ENDFOR
	\end{algorithmic}
\end{algorithm}

\begin{thm}[Regret Bound for CL-UCB \& CL-TS adapted from Chapter 19 in~\citet{lattimore2018bandit}] \label{thm:CL}
	Assume that $\left\|\theta\right\|_2 \leq 1$ and $\left\|f(\mathbf{Z})\right\|_2\leq 1$, the dimension of $\theta$ and $f(\mathbf{Z})$ are both $d$, then run CL-UCB with $\beta = 1+\sqrt{2\log\left(T\right)+d\log\left(1+\frac{T}{d}\right)}$ and CL-TS, the regret of CL-UCB and Bayesian regret of CL-TS can both be bounded by
	\begin{align*}
	\mathbb{E}\left[R_{T_{CL-UCB}}\right], BR_{T_{CL-TS}} 
	= \tilde{O}\left(d\sqrt{T}\right).
	\end{align*}
\end{thm}

\textbf{Remark: } Our algorithms are easily adapted to a more general setup, e.g. there exist a set of observable variables $\mathbf{W}$ that d-separates the manipulable variables and the reward variable and $P(\mathbf{W}|a)$ are known for all realizations $\mathbf{W}, a$. In this senario, one can replace $\text{Pa}_Y$ and $P(\text{Pa}_Y|a)$ in above algorithms with $\mathbf{W}$ and $P(\mathbf{W}|a)$ and achieve $\tilde{O}\left(\sqrt{|\mathbf{W}|T}\right)$ regret, where $|\mathbf{W}|$ refers to the number of realizations of $\mathbf{W}$. This is beneficial when $|\mathbf{W}|\ll\mathcal{A}$ or the reward's direct parents are not known nor observable, but the variables $\mathbf{W}$ are.

\section{LOWER BOUND FOR NON-CAUSAL METHODS} \label{sec:lower}
In this section, we show that it is necessary to use an algorithm that utilizes the causal structure.
We prove that there exists a simple bandit environment with causal information, for which the regret of the standard UCB algorithm scales \emph{at least} exponentially with the size $N$ of {\em all variables}.
For the same environment, our regret upper bounds of C-UCB and C-TS scale \emph{at most} exponentially with the size $n$ of {\em parents}.
Since it is possible to have $N \gg n$, this demonstrates the necessity of using causal bandits algorithms. 

We now describe the environment.
The bandit environment $\nu$ has $N$ variables $X_1,\ldots,X_N$, each can take a value from $\{1,2\}$. 
The marginal distribution for $X_i$ is $P(X_i=1) = p_i$, for $i=1,\ldots,N$. 
The reward node $Y$ is generated by $Y = \Delta X_1+\epsilon$, where $\Delta$ is a positive coefficient to be determined and $\epsilon \sim \mathcal{N}(0,1)$.
Actions are denoted by $do(X_1=i_1,\ldots,X_N=i_N)$, where $i_1,\ldots,i_N \in \{0,1,2\}$, and $0$ is an additional dimension for the case that we do not set any value for a variable.

In this example, there are three types of actions:
\begin{itemize}
	\item Type 0: Actions with $i_1=0$.
	\item Type 1: Actions with $i_1=1$.
	\item Type 2: Actions with $i_1=2$.
\end{itemize}
The expected reward for three types actions are $2\Delta-p_1\Delta$, $\Delta$ and $2\Delta$ respectively. Type 2 actions are optimal arms, while the gaps for type 0 and type 1 are $p_1\Delta$ and $\Delta$ respectively.

Now we present the lower bound of standard UCB for this environment.
\begin{thm}[Lower Bound for Standard UCB] \label{thm:UCB_lower}
	For any $\epsilon > 0$, there exists a constant $C_\epsilon > 0$ such that the following holds. In the bandit environment $\nu$ described above, running standard UCB algorithm for $T$ steps will incur regret at least $C_\epsilon \sqrt{3^N}T^{1/2-\epsilon}$.
\end{thm}

This theorem can be generalized to provide lower bounds for a broad class of MAB algorithms ($p$-order policies, see appendix), including standard TS.
We give a proof outline of this theorem.
The main idea is to apply Theorem~\ref{thm:finite_dep_lower} in~\citet{lattimore2018bandit}.
This is an algorithm-dependent lower bound that shows if an algorithm has a uniform regret upper bound for all instances in the unstructured bandit environment class (defined below), then it must have a particular instance-dependent regret lower bound.
We first show $\nu$ belongs to the unstructured bandit environment class. Next, since the standard UCB has a uniform  regret upper bound for this class, we can apply Theorem~\ref{thm:finite_dep_lower} in~\citet{lattimore2018bandit} to obtain a lower bound of standard UCB for $\nu$.

Now we give more details.
The unstructured Gaussian bandit environment class is defined as follows.
\begin{defini}[Unstructured Gaussian bandit environment class]\label{def:unstructured}
	A Gaussian bandit environment class $\mathcal{E}$ is unstructured if $\mathcal{A}$ is finite and there exists set of Gaussian distributions $\mathcal{M}_a:=\{\mathcal{N}(\mu,\sigma^2),\mu\in\mathbb{R},\sigma^2\leq 1\}$ for each $a\in\mathcal{A}$ such that
	\begin{align*}
	\mathcal{E} = \{\nu = (P_a:a\in\mathcal{A}):P_a\in\mathcal{M}_a, \forall a\in\mathcal{A}\}.
	\end{align*}
\end{defini}
Note this environment class is the Cartesian product over all distributions in $\mathcal{M}_a$ for each arm. 
This is a large class, and in particular it contains the environment $\nu$, which we formalize in the claim below.
\begin{claim}\label{claim:unstructured}
	Denote a unstructured $K$-arm Gaussian bandit environment class by $\mathcal{E}_K(\mathcal{N})$. Given any causal graph $\mathcal{G}$ and conditional probabilities $P(\text{Pa}_Y|a), \forall a\in\mathcal{A}$ where $Y$ is the reward variable and $\text{Pa}_Y$ are its parents, for any bandit instance $\nu'$ that satisfies:
	\begin{itemize}
		\item arms are $K$ interventions over a set of variables that are consistent with $\mathcal{G}$ and the corresponding conditional probabilities, and
		\item the conditional reward given parent values are independent Gaussian distributions:
		\begin{align*}	Y|_{\text{Pa}_Y=\mathbf{Z}} = \mathbb{E}\left[Y|\text{Pa}_Y=\mathbf{Z}\right]+\epsilon,
		\end{align*}
		where $\epsilon\sim\mathcal{N}(0,1)$,
	\end{itemize}
	we have that $\nu'\in\mathcal{E}_K(\mathcal{N})$.
\end{claim}
The proof of this claim is given in the appendix.
We can finish the proof of Theorem~\ref{thm:UCB_lower} by applying Theorem 16.4 in \cite{lattimore2018bandit} (presented in Theorem~\ref{thm:finite_dep_lower} in the appendix) on the environment $\nu$.


Note that Theorem~\ref{thm:finite_dep_lower} in~\cite{lattimore2018bandit} cannot be applied to causal algorithms.
Causal algorithms proposed in this paper can only perform well on environments equipped with the fixed input causal graph $\mathcal{G}$ and the corresponding conditional probabilities, and thus causal algorithms cannot provide uniform (sub-linear) regret upper bounds for all environments in the unstructured bandit environment class. 

\section{EXPERIMENTS}\label{sec:experiments}
We compare the performance of standard bandit with causal bandit algorithms to validate that causal information plays an important role in bandit algorithms. We also show that when the reward is truly generated by a linear combination of the reward's parent node, CL-TS and CL-UCB can further achieve smaller regrets comparing with C-TS and C-UCB that only use causal structures but not the linear property.

\subsection{PURE SIMULATION}
We set up a pure simulation environment that will allow us to run scaling experiments in order to qualitatively test the scaling predictions of our theory.
Throughout our pure simulations, we use a model in which there is a reward variable $Y$, reward's parent variables $W_1,\ldots,W_n$, taking values from $\{1,2\}$, and non-parent variables $X_1,\ldots,X_n$, taking values from $\{1,\ldots,m\}$. Reward $Y$ directly depends on its parent variables $W_1,\ldots,W_n$, while each parent variable $W_i$ directly depends on the corresponding non-parent variable $X_i$ ($i=1,\ldots,n$). The causal graph is displayed in Figure~\ref{fig:causal_graph_exp}.

\begin{figure}[tb]
	\centering
	\includegraphics[width=.25\textwidth]{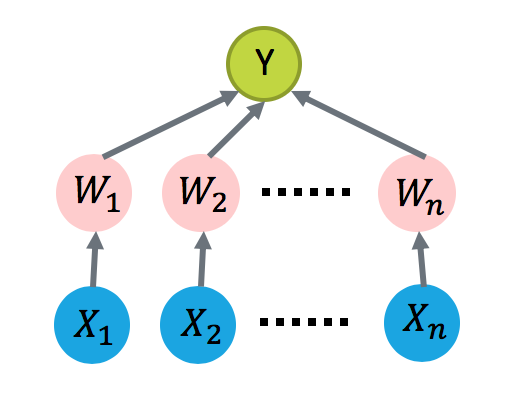}
	\caption{Causal Graph for Pure Simulation: only blue variables can be intervened.}
	\label{fig:causal_graph_exp}
\end{figure}

\textbf{Intervention set:} Denote an intervention by
\begin{align*}
	a=\text{do}(X_1=i_1,\ldots,X_n=i_n),
\end{align*}
where $i_1,\ldots,i_n \in \{0,1,\ldots,m\}$, $0$ is an additional dimension for the case that we do not set any value for a variable. That means only non-parent variables can be intervened, the parent variables of the reward are not under control.

Reward $Y$ is generated by:
$Y = \langle f(W_1,\ldots,W_n),\theta \rangle+\epsilon$,
where $f$ is a function applied on parent variables, $\theta$ is a $n$-dimensional vector, $\epsilon$ is a sub-gaussian random error.

\subsubsection{A Gentle Start: $m=3,n=4$}\label{sec:pure_simulation_fix}
We begin with a simple case where $m = 3, n=4$. The marginal distributions for $X_1,X_2,X_3,X_4$ and conditional probabilities for $W_i=1|X_i, i=1,\ldots,4$ are displayed in Table~\ref{table:psim_dist} (Section~\ref{sec:app_table}). 


For simplicity, we set
	$f(W_1,W_2,W_3,W_4) := (W_1,W_2,W_3,W_4)$,
and the error is a Gaussian variable $\epsilon \sim N(0,0.1^2)$.


\textbf{UCB algorithms}: The true linear coefficient $\theta$ is $(0.25,0.25,-0.25,-0.25)$. To approximate the expected regret, for each UCB algorithm we plot the average regret over 20 simulations.

\textbf{TS algorithms}: We plot both of the regret under $\theta=(0.25,0.25,-0.25,-0.25)$ and the Bayesian regret. For the frequentist one, the procedure is same as UCB algorithms described above. For the Bayesian one, the ``true" parameter $\theta$ is generated from its prior distribution $N(0,0.1^2I_4)$ for 20 times as Monte Carlo simulation. Then we plot the averaged regret over these 20 simulations to approximate the Bayesian regret.

Regret comparison plots are displayed in Figure~\ref{fig:Algo_comparison}.

\begin{figure*}[tb]
	\centering
	\includegraphics[width=.32\textwidth]{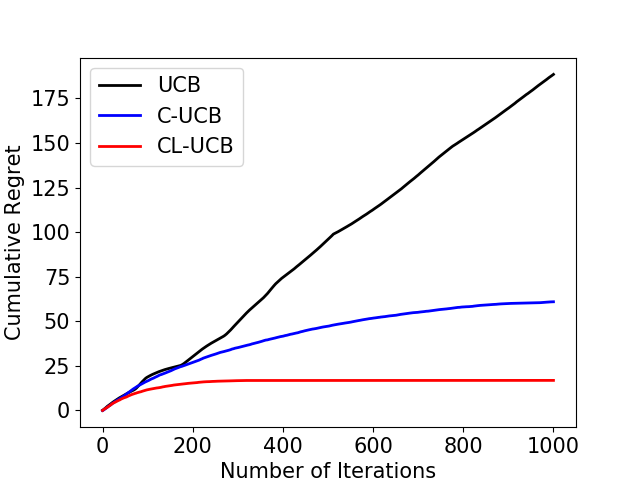}
	~
	\includegraphics[width=.32\textwidth]{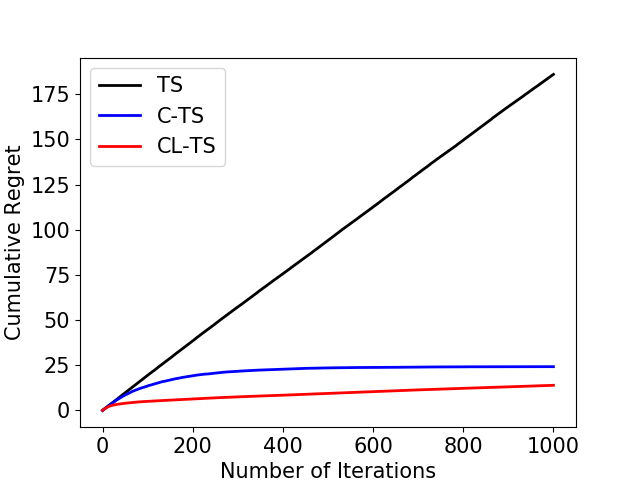}
	~
	\includegraphics[width=.32\textwidth]{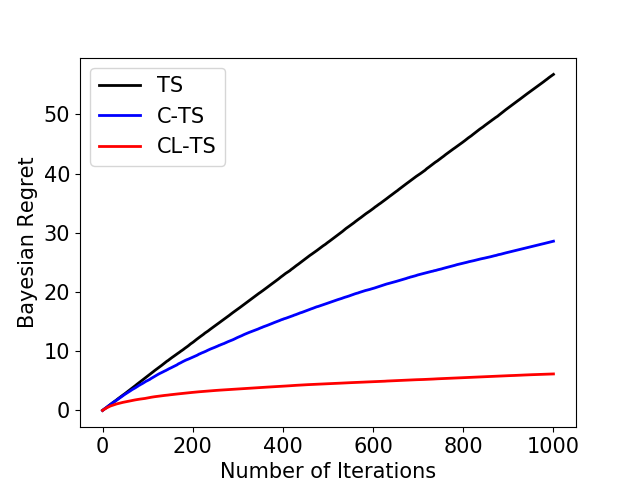}
	\caption{Regret comparison for $m=3, n=4$. Left: UCB regrets. Middle: TS regrets. Right: TS Bayesian regrets.}
	\label{fig:Algo_comparison}
\end{figure*}

\subsubsection{Scaling with Non-Parent Variables' Range: $\mathbf{m}$}\label{sec:pure_simulation_Xmax}
In this section, we fix $n=4$ while changing the domain range of non-parent variables $m$ from $2$ to $6$ and see how it affects the performance of all six algorithms.

In each simulation, the marginal probabilities for each non-parent variable $X_i$: $\{P(X_i=j)\}_{j=1}^m$ are generated from independent Dirichlet distributions with parameter $\alpha = \mathbb{1}_m$ and the conditional probabilities $P(W_i=1|X_i=j), i=1,\ldots,n;j=1,\ldots,m$ are generated randomly from $[0,1]$. Throughout we fix the $\theta=(0.25,0.25,-0.25,-0.25)$. For each algorithm, the final regret is averaged over 20 simulations. Regret comparison plot is displayed in Figure~\ref{fig:pure_simulation_Xmax}.

\begin{figure}[tb]
	\centering
	\includegraphics[width=.33\textwidth]{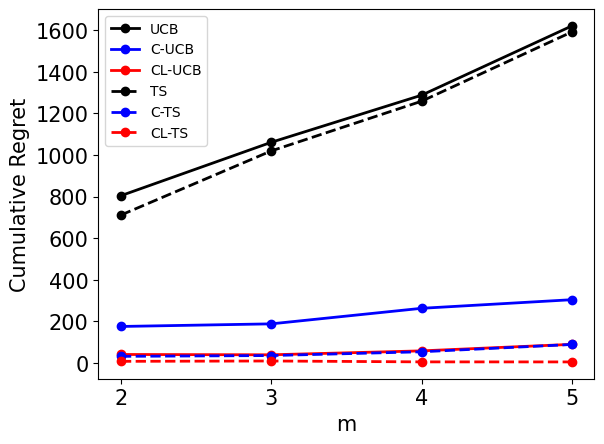}
	\caption{Cumulative regret v.s. $m$, fix $n=4$, time horizon $T=5000$.}
	\label{fig:pure_simulation_Xmax}
\end{figure}

\subsubsection{Scaling with Size of Parent Variables: $\mathbf{n}$}
In this section we fix $m=3$ while changing the number of parent/non-parent variables $n$ from $2$ to $6$. Since $X_i$ takes value from $\{1,2,3\}$ and $W_i$ takes value from $\{1,2\}$, by adding additional pair $W_i\sim X_i$, the intervention size increases much faster than the number of value assignments on parent variables. We compare the performance of six algorithms. 

In each simulation, the marginal probabilities for each non-parent variable $\{P(X_i=j)\}_{j=1}^m$ and conditional probabilities for each parent variable $P(W_i=1|X_i=j), j=1,\ldots,m$ are sampled in the same way as Section~\ref{sec:pure_simulation_Xmax}. 
To keep the reward at the same scale as $m$ varies, we use $\theta = (1,0,\ldots,0)$, where only the first element of linear coefficient is $1$ and other elements are all zeros.
For each algorithm, the final regret is averaged over 20 simulations. Regret comparison plot is displayed in Figure~\ref{fig:pure_simulation_k_T10000_N20_simple_coef}.


\begin{figure}[tb]
	\centering
	\includegraphics[width=.33\textwidth]{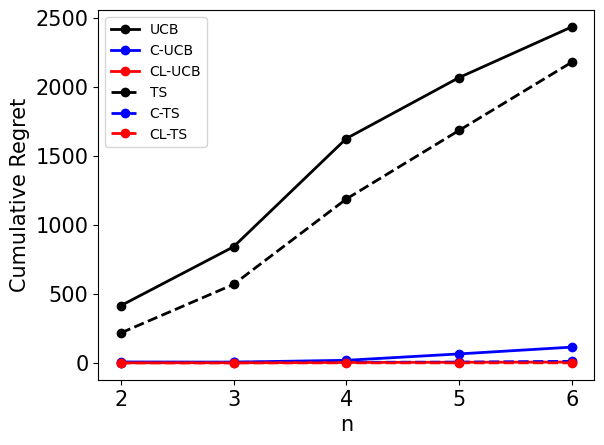}
	\caption{Cumulative regret v.s. $n$, fix $m=3$, time horizon $T=10000$.}
	\label{fig:pure_simulation_k_T10000_N20_simple_coef}
\end{figure}

\subsubsection{Conclusion of Pure Simulation}
In Figure~\ref{fig:Algo_comparison}, the left and middle plots demonstrate the performance of algorithms for a fixed causal bandit environment. We observe that for UCB and TS, causal linear algorithms outperform the ``non-linear" causal algorithms moderately and all causal algorithms outperform the standard bandit algorithms significantly. In the third plot, we demonstrate the performance in terms of Bayesian regret for three TS algorithms, and their performance order matches with the first two plots.

In Figure~\ref{fig:pure_simulation_Xmax}, we fix $n$ and the time horizon $T$ and compare the performance of the algorithm as $m$ increases. The regret of C-UCB, C-TS, CL-UCB and CL-TS do not vary as $m$ increases as their regret only depends on the size of parent variable value assignments. However, the regret of UCB and TS keeps increasing as $m$ grows. Thus, we validate that the performance of our causal algorithms are not affected by the number of interventions on non-parent variables.

In Figure~\ref{fig:pure_simulation_k_T10000_N20_simple_coef}, we fix $m$ and time horizon $T$ and compare the performance of all algorithms as $n$ grows. The regret of four causal algorithms does not vary a lot as $n$ increases. We show in our theorem that in worst case, the regret of C-TS and C-UCB grow with $\sqrt{k^n}$ and the regret of CL-TS and CL-UCB grow with $d$ for fixed time horizon. And we also observe in this simulation that for certain coefficient such as $\theta = (1,0,\ldots,0)$, the growth is even slower. Clearly the regret of standard UCB and TS algorithms keeps increasing as $n$ grows. 

\subsection{EMAIL CAMPAIGN DATA}
The experimental set up in this section is inspired by the email campaign data from Adobe.
The reward variable is binary: whether the commercial links inside the email are clicked or not by the recipient.
Features under control are ``product", such as Photoshop, Acrobat XI Pro, Adobe Stock, etc., ``purpose", such as awareness, promotion, operation, nurture, etc., ``send out time" that includes morning, afternoon and evening. Even though these features are highly correlated with the reward variable, but they are not the direct causes. The variables that are actually causing the email links clicking are: the subject length, two different email templates, send out time, so we set these variables as the reward's parents. 
The three features in blue that can be intervened are further connected with reward's parent variables as in Figure~\ref{fig:causal_graph_EC}. Each combination of product and purpose has an email pool, once they are fixed, the company picks out emails from the pool. Thus, subject length and email body template cannot be intervened, they depend on the emails picking out from the pool, which is a random process.

\begin{figure}[tb]
	\centering
	\includegraphics[width=.4\textwidth]{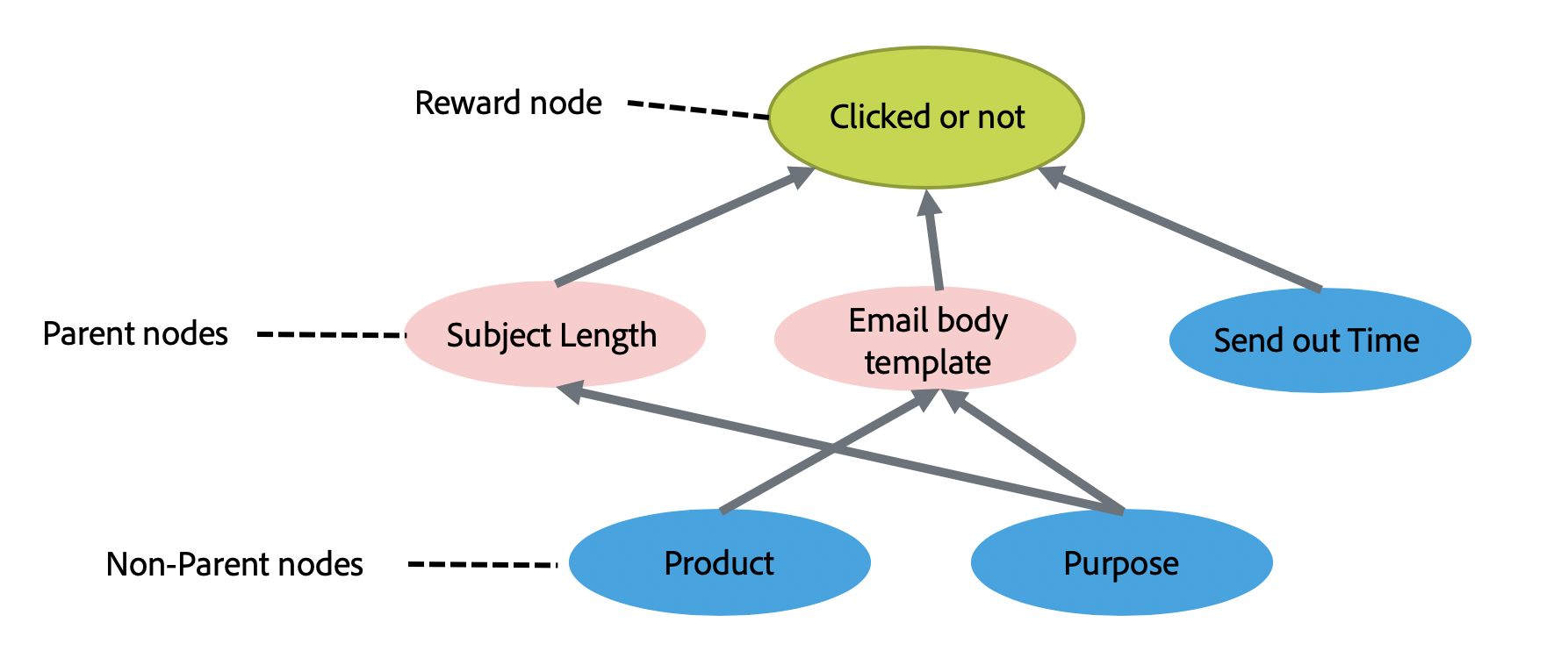}
	\caption{Causal Graph for Email Campaign: only blue nodes are under control.}
	\label{fig:causal_graph_EC}
\end{figure}

From historical knowledge, email with ``subject length" fewer than 7 words are more likely to be opened, so we denote ``subject length" by $Z_1$, taking values from $\{1,2\}$, representing ``less than 7 words" or not. ``Template" is denoted by $Z_2$, taking values from $\{1,2\}$, representing template indices ``1" or ``2". ``Send out time" is denoted by $Z_3$, taking values from $\{1,2,3\}$, representing ``morning", ``afternoon" and ``evening". 
We consider ``Photoshop" (1), ``Acrobat XI Pro"(2), ``Adobe Stock" (3) for the ``product" variable, denoted by $X_1$; ``Operational" (1), ``Promo" (2), ``Nurture" (3) and ``Awareness" (4) for purpose variable, denoted by $X_2$.

The marginal probabilities for $X_1$ and $X_2$ and $Z_3$, conditional distributions for $Z_1,Z_2$ are displayed in Table~\ref{table:email_dist} (Section~\ref{sec:app_table}). The reward follows a Bernoulli distribution, with parameter $1-(Z_1+Z_2+Z_3)/9$.


\textbf{Interventions} are denoted by $\text{do}(X_1=i_1,X_2=i_2,X_3=i_3)$, where $i_1,i_3 \in \{0,1,2,3\}$, $i_2\in \{0,1,2,3,4\}$, 0 means no intervention on a variable.

In Figure~\ref{fig:email_comparison}, we compare the performance of UCB, C-UCB, TS (beta prior) and C-TS (beta prior). We plot the average regret over 20 simulations to approximate the expected cumulative regret for each method. Clearly both of C-UCB and C-TS outperforms UCB and TS significantly. Besides, we observe that TS algorithms generally perform better than the UCB algorithms. This phenomenon is also consistent with previous empirical discoveries~\citep{chapelle2011empirical}.

\begin{figure}[h]
	\centering
	\includegraphics[width=0.33\textwidth]{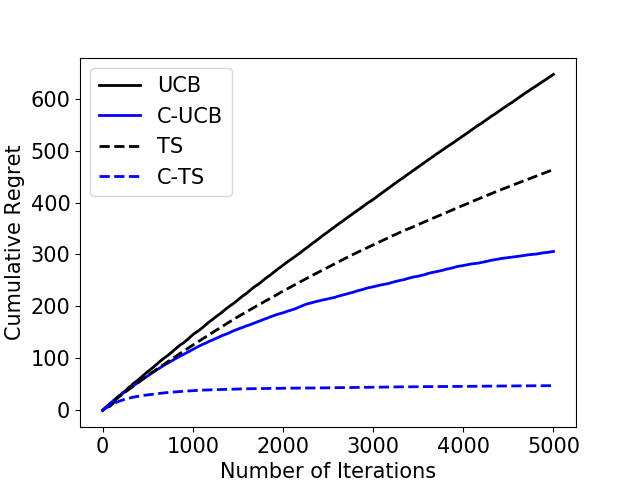}
	\caption{Regret Comparison of UCB, TS, C-UCB and C-TS for email campaign problem.}
	\label{fig:email_comparison}
\end{figure}

\section{DISCUSSION \& FUTURE WORK} \label{sec:diss}
We proposed C-UCB and C-TS algorithms and showed that their regret can be bounded by $\tilde{O}\left(\sqrt{k^nT}\right)$. We further extended linear bandit algorithms to their causal versions and showed the regret bound of CL-UCB and CL-TS can be reduced to $\tilde{O}\left(d\sqrt{T}\right)$. There are several interesting directions for future work.


\textbf{Extension to MDPs:} We plan to extend our causal bandit framework to the MDP (Markov decision process) setting. The key feature of causal MDP is that there is an additional dimension: \textit{state}, which can be affected by the previous intervention and the reward behaves differently under different status. This phenomenon is typical in many practical settings, including mobile health, online advertising and online education. 

\textbf{Learning causal structure:} In many cases the causal structure is not known beforehand or only partially understood. Therefore it is desirable to develop methods that can recover the underlying causal structure and minimize the cumulative regret at the same time. An ideal algorithm that can efficiently learn the causal structure and the bandit together should achieve lower regret than normal bandit algorithms when the time horizon $T$ is large. \citet{ortega2014generalized} empirically shows that TS can recover causal structures in some cases. Combining causal learning algorithm with those that minimize cumulative regret is an interesting direction to investigate.
\section*{ACKNOWLEDGEMENT}
Part of this work was done while Yangyi Lu was visiting Adobe. Ambuj Tewari would like to acknowledge the support of an Adobe Data Science Research Award, a Sloan Research Fellowship, and NSF grant CAREER IIS-1452099.

\bibliographystyle{apalike}
\bibliography{ref.bib}
\newpage
\onecolumn
\appendix
\section{Proof for Theorems}
We prove Theorem~\ref{thm:TS_bayes} before Theorem~\ref{thm:C-UCB}, since the former one includes more technical steps and main parts of the two proofs are similar. 
\subsection{Proof of Theorem~\ref{thm:TS_bayes} (C-TS)}\label{proof:CTS_bayes}
\begin{proof}
	By definition, $\mu_a := E\left[Y|a\right] = \sum_{i=1}^{k^n} E\left[Y|Pa_Y = Z_i\right]P\left(Pa_Y = Z_i|a\right)$, $a^* = \argmax_a \mu_a$.
	
	Define:
	\begin{align*}
	T_Z(t) &:= \sum_{s=1}^{t}\mathbb{1}_{\{Z_{(s)} = Z\}},\\
	\hat{\mu}_Z(t) &:= \frac{1}{T_Z(t)}\sum_{s=1}^{t}Y_s\mathbb{1}_{\{Z_{(s)} = Z\}},\\
	\mu_Z &:= E\left[Y|Pa_Y = Z\right],
	\end{align*}
	where $Z_{(s)}$ denotes the observed values of parent nodes for $Y$, in round $s$. Note that $\hat{\mu}_Z(t) = 0$ when $T_Z(t) = 0$.
	
	Let $E$ be the event that for all $t\in [T]$, $i\in[k^n]$ such that $\max_{a\in\mathcal{A}} P(Pa_Y=Z_i|a)>0$, we have
	\begin{align*}
	|\hat{\mu}_{Z_i}(t-1) - \mu_{Z_i}| \leq \sqrt{\frac{2log(1/\delta)}{1\vee T_{Z_i}(t-1)}}.
	\end{align*}
	
	For fixed $t$ and $i$, by Sub-Gaussian property, we can show
	\begin{align*}
	P\left(|\hat{\mu}_{Z_i}(t) - \mu_{Z_i}| \geq \sqrt{\frac{2\log(1/\delta)}{1\vee T_{Z_i}(t)}}\right) &= \mathbb{E}\left[P\left(|\hat{\mu}_{Z_i}(t) - \mu_{Z_i}| \geq \sqrt{\frac{2\log(1/\delta)}{1\vee T_{Z_i}(t)}}\middle| Z_{(1)},\ldots,Z_{(t)}\right)\right]\\
	&\leq \mathbb{E}\left[2\delta\right] = 2\delta.
	\end{align*}
	
	By union bound, we have $P\left(E^c\right) \leq 2\delta T k^n$. 
	
	The Bayesian regret can be written as
	\begin{align*}
	BR_T = \mathbb{E}\left[\sum_{t=1}^{T}\left(\mu_{a^*}-\mu_{a_t}\right)\right] = \mathbb{E}\left[\sum_{t=1}^{T}\mathbb{E}\left[\mu_{a^*}-\mu_{a_t}|\mathcal{F}_{t-1}\right]\right],
	\end{align*}
	where $\mathcal{F}_{t-1} = \sigma\left(a_1,Z_1,Y_1,\ldots,a_{t-1},Z_{t-1},Y_{t-1}\right)$.
	
	The key insight is to notice that by definition of Thompson Sampling, 
	\begin{align}\label{equ:TSinsight}
	P\left(a^*=\cdot|\mathcal{F}_{t-1}\right) = P\left(a_t=\cdot|\mathcal{F}_{t-1}\right).
	\end{align}
	
	Further, define $\text{UCB}_a(t):=\sum_{j=1}^{k^n}\text{UCB}_{Z_j}(t)P(Pa_Y=Z_j|a)$, we can bound the conditional expected difference between optimal arm and the arm played at round $t$ using equation~\ref{equ:TSinsight} by
	\begin{align*}
	&\mathbb{E}\left[\mu_{a^*}-\mu_{a_t}|\mathcal{F}_{t-1}\right] \\
	&= \mathbb{E}\left[\mu_{a^*}-\text{UCB}_{a_t}(t-1)+\text{UCB}_{a_t}(t-1)-\mu_{a_t}|\mathcal{F}_{t-1}\right]\\
	&= \mathbb{E}\left[\mu_{a^*}-\text{UCB}_{a^*}(t-1)+\text{UCB}_{a_t}(t-1)-\mu_{a_t}|\mathcal{F}_{t-1}\right].
	\end{align*}
	Next by tower rule, we have
	\begin{align*}
		BR_T = \mathbb{E}\left[\sum_{t=1}^{T}\left(\mu_{a^*}-\text{UCB}_{a^*}(t-1)+\text{UCB}_{a_t}(t-1)-\mu_{a_t}\right)\right].
	\end{align*}
	On event $E^c$, by the original definition of $BR_T$ we have $BR_T\leq 2T$. On event $E$, the first term is negative showing by the definition of $\text{UCB}_{\mathbf{Z}_j},j=1,\ldots,k^n$ and
	\begin{align*}
		\mu_{a^*}-\text{UCB}_{a^*}(t-1) = \sum_{j=1}^{k^n}\left(\mathbb{E}\left[Y|Pa_Y=Z_j\right]-\text{UCB}_{Z_j}(t-1)\right)P(Pa_Y=Z_j|a^*) \leq 0,
	\end{align*}
	because $\mathbb{E}\left[Y|Pa_Y=Z_j\right]-\text{UCB}_{Z_j}(t-1)\leq 0$ on event $E$.
	Also on event $E$, the second term can be bounded by
	\begin{align}
	 \mathbb{1}_E\sum_{t=1}^{T}\left(\text{UCB}_{a_t}(t-1)-\mu_{a_t}\right) &=\mathbb{1}_E\sum_{t=1}^{T}\sum_{j=1}^{k^n}\left(\text{UCB}_{Z_j}(t-1)-\mathbb{E}\left[Y|Pa_Y=Z_j\right]\right)P(Pa_Y=Z_j|a_t) \notag\\
	 &\leq \mathbb{1}_E\sum_{t=1}^{T} \sum_{j=1}^{k^n}\sqrt{\frac{8\log(1/\delta)}{1\vee T_{Z_j}(t-1)}}P(Pa_Y=Z_j|a_t) \notag\\
	 & \leq\mathbb{1}_E \sum_{t=1}^{T} \sum_{j=1}^{k^n}\sqrt{\frac{8\log(1/\delta)}{1\vee T_{Z_j}(t-1)}}\left(P(Pa_Y=Z_j|a_t)-\mathbb{1}_{\{Z_{(t)}=Z_j\}}+\mathbb{1}_{\{Z_{(t)}=Z_j\}}\right). \label{equ:TS_azuma}
	\end{align}
	The second part of equation~\ref{equ:TS_azuma} can be bounded by
	\begin{align*}
	\mathbb{1}_E\sum_{t=1}^{T} \sum_{j=1}^{k^n}\sqrt{\frac{8\log(1/\delta)}{1\vee T_{Z_j}(t-1)}}\mathbb{1}_{\{Z_{(t)}=Z_j\}} &\leq \mathbb{1}_E\sum_{j=1}^{k^n}\int_{0}^{T_{Z_j}(T)}\sqrt{\frac{8\log(1/\delta)}{s}} ds\\
	&\leq \sum_{j=1}^{k^n} \sqrt{32T_{Z_j}(T)\log(1/\delta)}\\
	&\leq \sqrt{32k^nT\log(1/\delta)}.
	\end{align*}
	For the first part of equation~\ref{equ:TS_azuma}, we define 
	$X_t:=\sum_{s=1}^{t} \sum_{j=1}^{k^n}\sqrt{\frac{8\log(1/\delta)}{1\vee T_{Z_j}(s-1)}}\left(P(Pa_Y=Z_j|a_s)-\mathbb{1}_{\{Z_{(s)}=Z_j\}}\right)$, $X_0:=0$. Note that $\{X_t\}_{t=0}^T$ is a martingale sequence and we have
	\begin{align*}
	|X_t-X_{t-1}|^2 &= \left|\sum_{j=1}^{k^n}\sqrt{\frac{8\log(1/\delta)}{1\vee T_{Z_j}(t-1)}}\left(P(Pa_Y=Z_j|a_t)-\mathbb{1}_{\{Z_{(t)}=Z_j\}}\right)\right|^2\\
	&\leq 32\log(1/\delta).
	\end{align*}
	By applying Azuma's inequality we have
	\begin{align*}
	P(|X_T|>\sqrt{k^nT\log(T)}\log(T)) \leq \exp\left(-\frac{k^n\log^3(T)}{32\log(1/\delta)}\right).
	\end{align*}
	We take $\delta = 1/T^2$, combine the first and second part of equation~\ref{equ:TS_azuma}, we show that with probability $1-P(E^c)-\exp\left(-\frac{k^n\log^2(T)}{64}\right) = 1-2k^n/T-\exp\left(-\frac{k^n\log^2(T)}{64}\right)$, 
	\begin{align*}
	R_T \leq 16\sqrt{k^nT\log(T)}\log(T).
	\end{align*}
	Thus the Bayesian regret can be bounded by:
	\begin{align*}
	\mathbb{E}\left[R_T\right] &\leq P(E^c)\times 2T+\exp\left(-\frac{k^n\log^2(T)}{64}\right)\times 2T+\sqrt{64k^nT\log(T)}\log(T)\\
	&\leq C\sqrt{k^nT\log(T)}\log(T).
	\end{align*}
	where $C$ is a constant and the above inequality holds for large $T$.
	Thus we have proved that $\mathbb{E}\left[R_T\right] = \tilde{O}\left(\sqrt{k^nT}\right)$.
\end{proof}

\subsection{Proof of Theorem~\ref{thm:C-UCB} (C-UCB)}\label{proof:CUCB}
\begin{proof}
	Let $E$ be the event that for all $t\in [T]$, $j\in [k^n]$, we have
	\begin{align*}
	\left|\hat{\mu}_{Z_j}(t-1) - \mathbb{E}\left[Y|Pa_Y=Z_j\right]\right| \leq \sqrt{\frac{2\log(1/\delta)}{1\vee T_{Z_j}(t-1)}}.
	\end{align*}
	Use same proof idea in Theorem~\ref{thm:TS_bayes}, we have $P(E^c) \leq 2\delta T k^n$. Define $\text{UCB}_a(t):=\sum_{j=1}^{k^n}\text{UCB}_{Z_j}(t)P(Pa_Y=Z_j|a)$, the regret can be rewritten as
	\begin{align*}
	R_T &= \sum_{t=1}^{T}(\mu_{a^*}-\mu_{a_t})\\
	&= \sum_{t=1}^{T} \left(\mu_{a^*}-\text{UCB}_{a_t}(t-1)+\text{UCB}_{a_t}(t-1)-\mu_{a_t}\right).
	\end{align*}
	On event $E^c$, $R_T\leq 2T$. On event $E$ we can show
	\begin{align*}
		\mu_{a^*}-\text{UCB}_{a_t}(t-1) &= \sum_{j=1}^{k^n}\mathbb{E}\left[Y|Pa_Y=Z_j\right]P(Pa_Y=Z_j|a^*)-\sum_{j=1}^{k^n}\text{UCB}_{Z_j}(t-1)P(Pa_Y=Z_j|a_t)\\
		&\leq \sum_{j=1}^{k^n}\text{UCB}_{Z_j}(t-1)P(Pa_Y=Z_j|a^*)-\sum_{j=1}^{k^n}\text{UCB}_{Z_j}(t-1)P(Pa_Y=Z_j|a_t) \leq 0,
	\end{align*}
	where the last inequality follows by the way to choose $a_t$ in Algorithm~\ref{algo:UCB_knownGraph}, the second last inequality follows by the definition of event $E$. Thus on event $E$ we have
	\begin{align}
		R_T &\leq \sum_{t=1}^{T}\left(\text{UCB}_{a_t}(t-1)-\mu_{a_t}\right) \notag\\
		&=\sum_{t=1}^{T}\sum_{j=1}^{k^n}\left(\text{UCB}_{Z_j}(t-1)-\mathbb{E}\left[Y|Pa_Y=Z_j\right]\right)P(Pa_Y=Z_j|a_t) \notag\\
		&\leq \sum_{t=1}^{T} \sum_{j=1}^{k^n}\sqrt{\frac{8\log(1/\delta)}{1\vee T_{Z_j}(t-1)}}P(Pa_Y=Z_j|a_t)\notag\\
		& \leq \sum_{t=1}^{T} \sum_{j=1}^{k^n}\sqrt{\frac{8\log(1/\delta)}{1\vee T_{Z_j}(t-1)}}\left(P(Pa_Y=Z_j|a_t)-\mathbb{1}_{\{Z_{(t)}=Z_j\}}+\mathbb{1}_{\{Z_{(t)}=Z_j\}}\right). \label{equ:UCB_azuma}
	\end{align}
	The second part of Equation~\ref{equ:UCB_azuma} can be bounded by
	\begin{align*}
		\sum_{t=1}^{T} \sum_{j=1}^{k^n}\sqrt{\frac{8\log(1/\delta)}{1\vee T_{Z_j}(t-1)}}\mathbb{1}_{\{Z_{(t)}=Z_j\}} &\leq \sum_{j=1}^{k^n}\int_{0}^{T_{Z_j}(T)}\sqrt{\frac{8\log(1/\delta)}{s}} ds\\
		&\leq \sum_{j=1}^{k^n} \sqrt{32T_{Z_j}(T)\log(1/\delta)}\\
		&\leq \sqrt{32k^nT\log(1/\delta)}.
	\end{align*}
	For the first part of equation~\ref{equ:UCB_azuma}, we define $X_t:=\sum_{s=1}^{t} \sum_{j=1}^{k^n}\sqrt{\frac{8\log(1/\delta)}{1\vee T_{Z_j}(s-1)}}\left(P(Pa_Y=Z_j|a_s)-\mathbb{1}_{\{Z_{(s)}=Z_j\}}\right)$, $X_0:=0$. Note that $\{X_t\}_{t=0}^T$ is a martingale sequence. 
	\begin{align*}
		|X_t-X_{t-1}|^2 &= \left|\sum_{j=1}^{k^n}\sqrt{\frac{8\log(1/\delta)}{1\vee T_{Z_j}(t-1)}}\left(P(Pa_Y=Z_j|a_t)-\mathbb{1}_{\{Z_{(t)}=Z_j\}}\right)\right|^2\\
		&\leq 32\log(1/\delta).
	\end{align*}
	By applying Azuma's inequality we have
	\begin{align*}
		P(|X_T|>\sqrt{k^nT\log(T)}\log(T)) \leq \exp\left(-\frac{k^n\log^3(T)}{32\log(1/\delta)}\right).
	\end{align*}
	We take $\delta = 1/T^2$, combine the first and second part of equation~\ref{equ:UCB_azuma}, with probability $1-P(E^c)-\exp\left(-\frac{k^n\log^2(T)}{64}\right) = 1-2k^n/T-\exp\left(-\frac{k^n\log^2(T)}{64}\right)$, the regret can be bounded by
	\begin{align*}
		R_T \leq 16\sqrt{k^nT\log(T)}\log(T).
	\end{align*}
	Thus the expected regret can be bounded by:
	\begin{align*}
		\mathbb{E}\left[R_T\right] &\leq P(E^c)\times 2T+\exp\left(-\frac{k^n\log^2(T)}{64}\right)\times 2T+\sqrt{64k^nT\log(T)}\log(T)\\
		&\leq C\sqrt{k^nT\log(T)}\log(T)
	\end{align*}
	where $C$ is a constant, above inequality holds for large $T$.
	Thus we prove $\mathbb{E}\left[R_T\right] = \tilde{O}\left(\sqrt{k^nT}\right)$
\end{proof}

\subsection{Proof of Theorem~\ref{thm:CL} (CL-TS)}
\begin{lemma}~\citep{lattimore2018bandit} \label{lemma:theta_conv}
	Notations same as algorithm~\ref{algo:CL-UCB} and algorithm~\ref{algo:CL-TS}. Let $\delta \in (0,1)$. Then with probability at least $1-\delta$ it holds that for all $t\in \mathbb{N}$,
	\begin{align*}
		\left\|\hat{\theta}_t-\theta\right\|_{V_t(\lambda)} \leq \sqrt{\lambda}\left\|\theta\right\|_2 + \sqrt{2\log\left(\frac{1}{\delta}\right)+\log\left(\frac{\det V_t(\lambda)}{\lambda^d}\right)}.
	\end{align*}
	Furthermore, if $\left\|\theta^*\right\|\leq m_2$, then $P(\exists t\in \mathbb{N}^+: \theta^* \notin \mathcal{C}_t)\leq \delta$ with
	\begin{align*}
		\mathcal{C}_t = \left\{\theta\in\mathbb{R}^d: \left\|\hat{\theta}_{t-1}-\theta\right\|_{V_{t-1}(\lambda)}\leq m_2\sqrt{\lambda}+\sqrt{2\log\left(\frac{1}{\delta}\right)+\log\left(\frac{\det V_{t-1}(\lambda)}{\lambda^d}\right)} \right\}.
	\end{align*}
\end{lemma}

\begin{lemma}~\citep{lattimore2018bandit} \label{lemma:vector_1}
	Let $x_1,\ldots,x_T \in \mathbb{R}^d$ be a sequence of vectors with $\left\|x_t\right\|_2\leq L < \infty$ for all $t\in [T]$, then
	\begin{align*}
		\sum_{t=1}^{T}\left(1\wedge\left\|x_t\right\|_{V_{t-1}^{-1}}^2\right) \leq 2\log\left(\det V_T\right) \leq 2d\log\left(1+\frac{TL^2}{d}\right),
	\end{align*}
	where $V_t = I_d + \sum_{s=1}^{t}x_sx_s^T$.
\end{lemma}
\begin{proof}
	W define $\beta = 1+\sqrt{2\log\left(T\right)+d\log\left(1+\frac{T}{d}\right)}$ and $V_t = I_d + \sum_{s=1}^tm_{a_s}m_{a_s}^T$ same as Algorithm~\ref{algo:CL-TS}, where $m_{a}:=\sum_{i=1}^{k^n}f(Z_i)P(Pa_Y=Z_i|a)$.
	Define upper confidence bound $\text{UCB}_t:\mathcal{A}\rightarrow \mathbb{R}$ by 
	\begin{align*}
		\text{UCB}_t(a) = \max_{\theta\in\mathcal{C}_t}\langle \theta,m_a\rangle =  <\hat{\theta}_{t-1},m_a>+\beta \left\|m_a\right\|_{V_{t-1}^{-1}},
	\end{align*}
	where $\mathcal{C}_t = \left\{\theta\in \mathbb{R}^d: \left\|\theta-\hat{\theta}_{t-1}\right\|_{V_{t-1}} \leq \beta \right\}$. By Lemma~\ref{lemma:theta_conv}, we have
	\begin{align*}
		P\left(\exists t\leq T: \left\|\hat{\theta}_{t-1}-\theta\right\|_{V_{t-1}} \geq 1+\sqrt{2\log\left(T\right)+\log\left(\det V_t\right)}\right) \leq \frac{1}{T}.
	\end{align*}
	And note $\left\|m_a\right\|_2\leq 1$, thus by geometric means inequality we have 
	\begin{align*}
		\det V_t \leq \left(trace(\frac{V_t}{d})\right)^d \leq \left(1+\frac{T}{d}\right)^d.
	\end{align*}
	Thus, by $\left\|\theta\right\|_2\leq 1$,
	\begin{align*}
		P\left(\exists t\leq T: \left\|\hat{\theta}_{t-1}-\theta\right\|_{V_{t-1}} \geq 1+\sqrt{2\log\left(T\right)+d\log\left(1+\frac{T}{d}\right)}\right) \leq \frac{1}{T}.
	\end{align*}
	Let $E_t$ be the event that $\left\|\hat{\theta}_{t-1}-\theta\right\|_{V_{t-1}} \leq \beta$, $E:=\cap_{t=1}^T E_t$, $a^*:=\argmax_a \sum_{i=1}^{k^n}\langle f(Z_i),\theta\rangle P(Pa_Y=Z_i|a)$, which is a random variable in this setting because $\theta$ is random. Then
	\begin{align}
		BR_T &= \mathbb{E}\left[\sum_{t=1}^{T}\left\langle\sum_{i=1}^{k^n}f(Z_i)\left(P\left(Pa_Y=Z_i|a^*\right)-P\left(Pa_Y=Z_i|a_t\right)\right),\theta \right\rangle\right] \notag\\
		&= \mathbb{E}\left[\mathbb{1}_{E^c}\sum_{t=1}^{T}\left\langle\sum_{i=1}^{k^n}f(Z_i)\left(P(Pa_Y=Z_i|a^*)-P(Pa_Y=Z_i|a_t)\right),\theta \right\rangle\right] \notag\\
		&+\mathbb{E}\left[\mathbb{1}_E\sum_{t=1}^{T}\left\langle\sum_{i=1}^{k^n}f(Z_i)\left(P(Pa_Y=Z_i|a^*)-P(Pa_Y=Z_i|a_t)\right),\theta \right\rangle\right] \notag\\
		&\leq 2TP(E^c)+\mathbb{E}\left[\mathbb{1}_E\sum_{t=1}^{T}\left\langle\sum_{i=1}^{k^n}f(Z_i)\left(P(Pa_Y=Z_i|a^*)-P(Pa_Y=Z_i|a_t)\right),\theta \right \rangle\right]\notag\\
		& \leq 2+ \mathbb{E}\left[\sum_{t=1}^{T}\mathbb{1}_{E_t}\left\langle\sum_{i=1}^{k^n}f(Z_i)\left(P(Pa_Y=Z_i|a^*)-P(Pa_Y=Z_i|a_t)\right),\theta \right\rangle\right]. \label{equ:Blinear_sec}
	\end{align}
	Again, we know from equation~\ref{equ:TSinsight} such that $P(a^*=\cdot|\mathcal{F}_{t-1}) = P(a_t=\cdot|\mathcal{F}_{t-1})$, where $\mathcal{F}_{t-1} = \sigma(Z_1,a_1,Y_1,\ldots,Z_{t-1},a_{t-1},Y_{t-1})$. Thus we have
	\begin{align*}
		&\mathbb{E}\left[\mathbb{1}_{E_t}\left\langle \sum_{i=1}^{k^n}f(Z_i)\left(P\left(Pa_Y=Z_i|a^*\right)-P\left(Pa_Y=Z_i|a_t\right)\right),\theta\right\rangle\middle|\mathcal{F}_{t-1}\right]\\
		=& \mathbb{1}_{E_t}\mathbb{E}\left[\left\langle \sum_{i=1}^{k^n}f(Z_i)\left(P\left(Pa_Y=Z_i|a^*\right)-P\left(Pa_Y=Z_i|a_t\right)\right),\theta\right\rangle\middle|\mathcal{F}_{t-1}\right]\\
		=& \mathbb{1}_{E_t}\mathbb{E}\left[\left\langle \sum_{i=1}^{k^n}f(Z_i)P\left(Pa_Y=Z_i|a^*\right),\theta\right\rangle - UCB_t(a^*)+UCB_t(a_t)-\left\langle \sum_{i=1}^{k^n}f(Z_i)P(Pa_Y=Z_i|a_t),\theta\right\rangle\middle|\mathcal{F}_{t-1}\right]\\
		\leq & \mathbb{1}_{E_t}\mathbb{E}\left[UCB_t(a_t)-\left\langle \sum_{i=1}^{k^n}f(Z_i)P(Pa_Y=Z_i|a_t),\theta\right\rangle\middle|\mathcal{F}_{t-1}\right]\\
		\leq & \mathbb{1}_{E_t}\mathbb{E}\left[\left\langle \sum_{i=1}^{k^n}f(Z_i)P(Pa_Y=Z_i|a_t),\hat{\theta}_{t-1}-\theta\right\rangle\middle|\mathcal{F}_{t-1}\right]+\beta \left\|\sum_{i=1}^{k^n}f(Z_i)P(Pa_Y=Z_i|a)\right\|_{V_{t-1}^{-1}}\\
		\leq & 2\beta \left\|\sum_{i=1}^{k^n}f(Z_i)P(Pa_Y=Z_i|a)\right\|_{V_{t-1}^{-1}}.
	\end{align*}
	Substituting into the second term of equation~\ref{equ:Blinear_sec},
	\begin{align*}
		&\mathbb{E}\left[\sum_{t=1}^{T}\mathbb{1}_{E_t}\left\langle\sum_{i=1}^{k^n}f(Z_i)\left(P(Pa_Y=Z_i|a^*)-P(Pa_Y=Z_i|a_t)\right),\theta \right\rangle\right] \\
		\leq & 2\mathbb{E}\left[\beta\sum_{t=1}^{T}\left(1\wedge\left\|\sum_{i=1}^{k^n}f(Z_i)P(Pa_Y=Z_i|a)\right\|_{V_{t-1}^{-1}}\right)\right] \\
		\leq & 2\sqrt{T\mathbb{E}\left[\beta^2\sum_{t=1}^{T}\left(1\wedge\left\|\sum_{i=1}^{k^n}f(Z_i)P(Pa_Y=Z_i|a)\right\|_{V_{t-1}^{-1}}^2\right)\right]}\ \text{( By Cauchy-Schwartz)}\\
		\leq & 2\sqrt{2dT\beta^2\log\left(1+\frac{T}{d}\right)} \ \text{(By Lemma~\ref{lemma:vector_1})}.
	\end{align*}
	Putting together we prove
	\begin{align}
		BR_T \leq 2+2\sqrt{2dT\beta^2\log\left(1+\frac{T}{d}\right)} = \tilde{O}\left(d\sqrt{T}\right).
	\end{align}
\end{proof}

\subsection{Proof of Theorem~\ref{thm:CL} (CL-UCB)}
\begin{proof}
	Define $\beta = 1+\sqrt{2\log\left(T\right)+d\log\left(1+\frac{T}{d}\right)}$, by Lemma~\ref{lemma:theta_conv} and above proof for CL-TS we have
	\begin{align*}
		&P(\exists t\leq T:\left\|\hat{\theta}_{t-1}-\theta^*\right\|_{V_{t-1}}\geq \beta)\leq \frac{1}{T},\\
		&P(\exists t\in \mathbb{N}^+:\theta^*\notin \mathcal{C}_t)\leq \frac{1}{T},
	\end{align*}
	where $\mathcal{C}_t = \left\{\theta\in \mathbb{R}^d: \left\|\theta-\hat{\theta}_{t-1}\right\|_{V_{t-1}} \leq \beta \right\}$.
	
	Let $\tilde{\theta}_t$ denote a $\theta$ that satisfies $\langle \tilde{\theta}_t,a_t\rangle = UCB_t(a_t)$. Again let $E_t$ be the event that $\left\|\hat{\theta}_{t-1}-\theta^*\right\|_{V_{t-1}}\leq \beta$, let $E=\bigcap E_t$, $a^* = \argmax_a \sum_{j=1}^{k^n}\langle f(Z_j),\theta \rangle P(Pa_Y=Z_j|a)$. Then on event $E_t$, using the fact that $\theta^* \in \mathcal{C}_t$ we have
	\begin{align*}
		\langle\theta^*,\sum_{j=1}^{k^n}f(Z_j)P(Pa_Y=Z_j|a^*)\rangle \leq UCB_t(a^*) \leq UCB_t(a_t) = \langle \tilde{\theta}_t,\sum_{j=1}^{k^n}f(Z_j)P(Pa_Y=Z_j|a_t)\rangle 
	\end{align*}
	Thus we can bound the difference of expected reward between optimal arm and $a_t$ by
	\begin{align*}
		\mu_{a^*}-\mu_{a_t} &= \langle \theta^*,\sum_{j=1}^{k^n}f(Z_j)P(Pa_Y=Z_j|a^*)\rangle - \langle \theta^*,\sum_{j=1}^{k^n}f(Z_j)P(Pa_Y=Z_j|a_t)\rangle\\
		&\leq \langle \tilde{\theta}_t-\theta^*,\sum_{j=1}^{k^n}f(Z_j)P(Pa_Y=Z_j|a_t)\rangle\\
		&\leq 2 \wedge 2\beta\left\|\sum_{j=1}^{k^n}f(Z_j)P(Pa_Y=Z_j|a_t)\right\|_{V_{t-1}^{-1}}\\
		& \leq 2\beta\left(1\wedge \left\|\sum_{j=1}^{k^n}f(Z_j)P(Pa_Y=Z_j|a_t)\right\|_{V_{t-1}^{-1}}\right).
	\end{align*}
	So the expected regret can be further bounded by:
	\begin{align*}
		\mathbb{E}\left[R_T\right] &= \mathbb{E}\left[\sum_{t=1}^{T}(\mu_{a^*}-\mu_{a_t})\right]= \mathbb{E}\left[\mathbb{1}_E\sum_{t=1}^{T}(\mu_{a^*}-\mu_{a_t})\right]+\mathbb{E}\left[\mathbb{1}_{E^c}\sum_{t=1}^{T}(\mu_{a^*}-\mu_{a_t})\right]\\
		&\leq \mathbb{E}\left[\sum_{t=1}^{T}(\mu_{a^*}-\mu_{a_t})\mathbb{1}_{E_t}\right]+\mathbb{E}\left[\mathbb{1}_{E^c}\sum_{t=1}^{T}(\mu_{a^*}-\mu_{a_t})\right]\\
		&\leq 2\beta\sum_{t=1}^{T}\left(1\wedge \left\|\sum_{j=1}^{k^n}f(Z_j)P(Pa_Y=Z_j|a_t)\right\|_{V_{t-1}^{-1}}\right)+2TP(E^c) \\
		&\leq 2+2\beta\sqrt{T\sum_{t=1}^{T}\left(1\wedge \left\|\sum_{j=1}^{k^n}f(Z_j)P(Pa_Y=Z_j|a_t)\right\|_{V_{t-1}^{-1}}^2\right)}\ \text{(By Cauchy-Schwartz)}\\
		&\leq 2+2\beta \sqrt{2dT\log\left(1+\frac{T}{d}\right)} \ ~\text{(By Lemma~\ref{lemma:vector_1})}
	\end{align*}
\end{proof}

\subsection{Proof of Claim~\ref{claim:unstructured}}
\begin{proof}
	Denote the reward variable for action $a$ by $Y|_a$ and denote the reward variable given fixed parent values by $Y|_{\text{Pa}_Y=\mathbf{Z}}$. According to the causal information, $Y|_a$ can be represented as a weighted sum of $Y|_{\text{Pa}_Y=\mathbf{Z}}$:
	\begin{align}
		Y|_a = \sum_{\mathbf{Z}}P(\text{Pa}_Y=\mathbf{Z}|a)Y|_{\text{Pa}_Y=\mathbf{Z}}.
	\end{align}
	In the statement of claim~\ref{claim:unstructured} we know that $Y|_{\text{Pa}_Y=\mathbf{Z}}$ are independent Gaussian distributions, therefore $Y|_a$, a weighted sum of Gaussian distributions still follows a Gaussian distribution. It remains to show the variance of $Y|_a$ is less than 1.
	\begin{align}
	    \text{Var}(Y|_a) &= \sum_{\mathbf{Z}}P(\text{Pa}_Y=\mathbf{Z}|a)^2\text{Var}(Y|_{\text{Pa}_Y=\mathbf{Z}})\\
	    &\leq \sum_{\mathbf{Z}}P(\text{Pa}_Y=\mathbf{Z}|a)^2\leq \sum_{\mathbf{Z}}P(\text{Pa}_Y=\mathbf{Z}|a) = 1,
	\end{align}
	where the first inequality above uses the condition that $\text{Var}(Y|_{\text{Pa}_Y=\mathbf{Z}})\leq 1$.
	We show that the reward for every arm $Y|_a$ is Gaussian distributed with variance less than 1, thus the bandit environment $\nu'$ described in the claim is an instance in Gaussian bandit environment class.
\end{proof}

\subsection{Proof of Theorem~\ref{thm:UCB_lower}}
We first introduce an important concept.

\begin{defini}[$p$-order Policy]
	For K-arm unstructured Gaussian bandit environments $\mathcal{E} := \mathcal{E}_K(\mathcal{N})$ and policy $\pi$, whose regret, on any $\nu \in \mathcal{E}$, is bounded by $CT^p$ for some $C>0$ and $p>0$. We call this policy class $\Pi(\mathcal{E},C,T,p)$, the class of p-order policies. 
\end{defini}
Note that UCB and TS are in this class with $C = C'_\epsilon \sqrt{K}$ and $p=1/2+\epsilon$ with some $C'_\epsilon>0$ for arbitrary small $\epsilon$.

We use the following result to prove our theorem.
\begin{thm}[Finite-time, instance-dependent regret lower bound for $p$-order policies, Theorem 16.4 in~\citet{lattimore2018bandit}]
	\label{thm:finite_dep_lower}
	Let $\nu\in\mathcal{E}_K(\mathcal{N})$ be a $K$-arm Gaussian bandit with mean vector $\mu\in\mathbb{R}^K$ and suboptimality gaps $\Delta\in[0,\infty)^K$. Let 
	\begin{align*}
		\mathcal{E}(\nu) = \{\nu'\in\mathcal{E}_K(\mathcal{N}):\mu_i(\nu')\in[\mu_i,\mu_i+2\Delta_i]\}.
	\end{align*}
	Suppose $\pi$ is a $p$-order policy such that $\exists C>0$ and $p\in(0,1)$, $R_T(\pi,\nu')\leq CT^p$ for all $T$ and $\nu'\in\mathcal{E}(\nu)$. Then for any $\epsilon\in(0,1]$,
	\begin{align*} 
	\mathbb{E}R_T(\pi,\nu) \geq \frac{2}{(1+\epsilon)^2} \sum_{i:\Delta_i>0}\left(\frac{(1-p)\log(T)+\log(\frac{\epsilon\Delta_i}{8C})}{\Delta_i}\right)^+,
	\end{align*}
	where $(x)^+ = \max(x,0)$ is the positive part of $x \in \mathbb{R}$. 
\end{thm}
\begin{proof}[Proof of Theorem~\ref{thm:UCB_lower}]
	Consider the bandit environment $\nu$ described in section~\ref{sec:lower}. By claim~\ref{claim:unstructured} we know $\nu$ is an instance in unstructured Gaussian bandit environment class, so we can further apply Theorem~\ref{thm:finite_dep_lower}. The size of three types of actions are all $3^N/3$. For Type 1 actions, its gap compared to the optimal actions is $\Delta$, for Type 0 actions, gap is $p_1\Delta$.
	Plugging into the results of Theorem~\ref{thm:finite_dep_lower}, for every $p$-order policy over $\mathcal{E}(\nu)$, we have
	\begin{align}\label{equ:lower_bound_middle}
	\mathbb{E}R_T(\pi,\nu) \geq \frac{1}{2} \frac{3^N}{3}\left(\frac{(1-p)\log(T)+\log(\frac{\Delta}{8C})}{\Delta}\right)^+ + \frac{1}{2} \frac{3^N}{3}\left(\frac{(1-p)\log(T)+\log(\frac{p_1\Delta}{8C})}{p_1\Delta}\right)^+.
	\end{align}
	In particular, choose $\Delta = 8\rho C T^{p-1}$, we get
	\begin{align*}
	(1-p)\log(T)+\log(\frac{\Delta}{8C}) &= \log(\rho),\\
	(1-p)\log(T)+\log(\frac{p_1\Delta}{8C}) &= \log(p_1\rho).
	\end{align*}
	Note that $\sup_{\rho >0} \log(\rho)/\rho = \exp(-1) \approx 0.35$, and we next plug above two equations in Equation~\ref{equ:lower_bound_middle} to get
	\begin{align*}
	\mathbb{E}R_T(\pi,\nu) \geq \frac{3^N}{3} \frac{0.35}{8CT^{p-1}}.
	\end{align*}
	Now consider $\pi$ to be UCB, by plugging in $C = C'_\epsilon \sqrt{3^N}$ and $p = 1/2+\epsilon$ we have
	\begin{align*}
	\mathbb{E}R_T(UCB,\nu) \geq \frac{0.35}{24C'_\epsilon} \sqrt{3^N} T^{1/2-\epsilon}.
	\end{align*}
\end{proof}

\section{Probability Tables Used in Experiments}\label{sec:app_table}
\begin{table}[h]
	\centering
	\begin{tabular}{c c c c}
		\hline
		    $i$       & $1$& $2$ & $3$\\
		\hline
		$P(X_1=i)$ & 0.3 & 0.4 & 0.3\\
		$P(X_2=i)$ & 0.3 & 0.3 & 0.4\\
		$P(X_3=i)$ & 0.5 & 0.3 & 0.2\\
		$P(X_4=i)$ & 0.25 & 0.25 & 0.5\\
		$P(W_1=1|X_1=i)$ & 0.2 & 0.5 & 0.8\\
		$P(W_2=1|X_2=i)$ & 0.3 & 0.2 & 0.8\\
		$P(W_3=1|X_3=i)$ & 0.4 & 0.6 & 0.5\\
		$P(W_4=1|X_4=i)$ & 0.3 & 0.5 & 0.6\\
		\hline
	\end{tabular}
\caption{Marginal and conditional probabilities for pure simulation experiment in section~\ref{sec:pure_simulation_fix}, numbers are randomly selected.}
\label{table:psim_dist}
\end{table}
\begin{table}[h]
	\centering
	\begin{tabular}{c c c c c}
		\hline
		$i$ & $1$ & $2$ & $3$ & $4$\\
		\hline
		$P(X_1=i)$ & 0.2 & 0.2 & 0.6 & \\
		$P(X_2=i)$ & 0.05 & 0.6 & 0.3 & 0.05\\
		$P(Z_3=i)$ & 0.5 & 0.2 & 0.3 & \\
		$P(Z_1=1|X_2=i)$ & 0.7 & 0.7 & 0.3 & 0.3\\
		$P(Z_2=1|X_1=3,X_2=i)$ & 0.6 & 0.7 & 0.6 & 0.5\\
		$P(Z_2=1|X_1\neq 3,X_2=i)$ & 0.8 & 0.9 & 0.5 & 0.2\\
		\hline
	\end{tabular}
\caption{Marginal and conditional probabilities for email campaign causal graph.}
\label{table:email_dist}
\end{table}

\end{document}